\documentclass[10pt]{article}

\usepackage{graphicx}
\usepackage{color}

\usepackage[authoryear]{natbib}

\usepackage{amsmath,amssymb,bm}
\usepackage{mathtools}
\mathtoolsset{showonlyrefs=true}

\newcommand{\R}{\mathbb{R}}

\newcommand{\PP}{\mathbb{P}}
\DeclareMathOperator*{\argmax}{argmax}

\DeclareMathOperator*{\argsort}{arg\,sort}

\newcommand{\gA}{\mathcal{A}} 
\newcommand{\gN}{\mathcal{N}} 
\newcommand{\gM}{\mathcal{M}} 
\newcommand{\gQ}{\mathcal{Q}} 
\newcommand{\gZ}{\mathcal{Z}} 

\newcommand{\vY}{\bm{Y}}            
\newcommand{\vy}{\bm{y}}            

\newcommand{\vmu}{\bm{\mu}}
\newcommand{\veps}{\bm{\varepsilon}}

\newcommand{\veta}{\bm{\eta}}

\newcommand{\va}{\bm{a}}
\newcommand{\vb}{\bm{b}}
\newcommand{\vbeta}{\bm{\beta}}

\newcommand{\mX}{\bm{X}}           
\newcommand{\mI}{\bm{I}}
\newcommand{\mSigma}{\bm{\Sigma}}

\newcommand{\pselective}{p_{\mathrm{selective}}}
\newcommand{\pnaive}{p_{\mathrm{naive}}}
\newcommand{\num}{\mathrm{num}}
\newcommand{\den}{\mathrm{den}}
\newcommand{\abs}[1]{\left\lvert #1 \right\rvert}
\newcommand{\norm}[1]{\left\lVert #1 \right\rVert}

\usepackage{amsthm}
\theoremstyle{plain}
\newtheorem{theorem}{Theorem}

\newtheorem{lemma}{Lemma}
\theoremstyle{definition}

\newtheorem{condition}{Condition}

\usepackage[hyphens]{url}
\usepackage{booktabs}
\usepackage{caption}
\usepackage{subcaption}
\usepackage{algorithm}
\usepackage{algorithmic}

\usepackage{float}
\usepackage{listings}
\floatstyle{ruled}
\newfloat{listing}{tb}{lst}
\floatname{listing}{Listing}
\usepackage{etoolbox}
\AtBeginEnvironment{listing}{}

\usepackage{tikz}
\usetikzlibrary{arrows.meta}

\usepackage[hidelinks]{hyperref}

\title{\vspace{-25mm}Automatic Statistical Test \\
for Rationally Expressible Algorithms \\ by Selective Inference, \\with Applications to Feature Selection}

\date{\today}

\makeatletter
\def\@fnsymbol#1{\ensuremath{\ifcase#1\or
{1}\or
{2}\or
{3}\or
{4}\or
{5}\or
{\dagger}\or
\else\@ctrerr\fi}}
\makeatother

\author{
Teruyuki Katsuoka\thanks{Nagoya University} ,
Tomohiro Shiraishi\footnotemark[1] \thanks{RIKEN} ,\\
Shuichi Nishino\footnotemark[1] \footnotemark[2] ,
Ichiro Takeuchi\footnotemark[1] \footnotemark[2] \thanks{Corresponding author. e-mail: takeuchi.ichiro.n6@f.mail.nagoya-u.ac.jp}
}
\begin{document}

\maketitle

\thispagestyle{empty}

\begin{abstract}
    \noindent
    {\linespread{1}\selectfont
    Selective inference (SI) provides statistically valid $p$-values for hypotheses selected by applying an algorithm to the data, correcting for the bias that arises when the same data are used both to select and to test a hypothesis.
Developing an SI procedure for a new algorithm, however, has required an expert to derive, and then implement, the \emph{selection event}, i.e., the conditions under which the hypothesis is selected.
Repeating this specialized effort for every new algorithm is why exact SI has so far been available for only a narrow class.
We propose \emph{AutoSI}, a framework that removes this barrier in two ways.
First, AutoSI constructs the selection event automatically from the algorithm's individual operations, so the user only writes the algorithm as ordinary NumPy-like code and derives nothing by hand.
Second, AutoSI broadens the class of selection events SI can handle: existing exact methods are limited to selection events characterized by linear or quadratic inequalities in the data, whereas AutoSI covers any algorithm expressible through rational functions of the data (ratios of polynomials).
We prove that the $p$-values computed by AutoSI are exactly valid in finite samples.
We demonstrate AutoSI on three feature-selection methods, each written in a few dozen lines of code.
One of these methods, the lasso with its tuning parameter selected by cross-validated $R^2$, cannot be handled within existing exact SI frameworks and is made possible by AutoSI.
Experiments on synthetic and real datasets show that the resulting $p$-values control the type I error rate (i.e., the false positive rate) at the nominal level while retaining high power.
\par}
\end{abstract}

\newpage
\section{Introduction}
\label{sec:introduction}
Modern data analysis increasingly relies on algorithms to search over many possible findings and select those that appear most promising.
Once such a data-driven finding is selected, a basic question is whether it reflects a genuine signal or merely an artifact of noise in the data.
A traditional statistical test cannot answer this question correctly, because the same data are used twice, first to choose the hypothesis and then to test it~\citep{fithian2014optimal,benjamini2020selective}.
When a hypothesis is singled out precisely because it appears promising in the data, testing it on those same data overstates its significance and yields far too many false positives.
Reusing the same data for both selection and testing---a practice commonly known as \emph{double dipping}---is a well-recognized source of irreproducible findings~\citep{benjamini2020selective}, and it is becoming more pressing as automated pipelines select features and models with little human supervision.

Selective inference (SI) is a statistical framework that provides valid inference even when the hypothesis is chosen by looking at the data~\citep{lee2016exact,tibshirani2016exact,fithian2014optimal,taylor2015statistical}.
Its key idea is to account for the selection explicitly.
Rather than pretending the hypothesis was fixed in advance, SI tests it under the condition that the algorithm would have selected it from the data.
Conditioning on this fact corrects for the bias introduced by the selection, so that the resulting $p$-value reliably controls the type I error rate.
First introduced for feature selection in linear models~\citep{lee2014exact,lee2016exact,tibshirani2016exact}, SI has since become a principled way to attach statistical guarantees to data-driven findings across many domains~\citep{gao2022selective,jewell2022testing,miwa2023valid,shiraishi2024statistical}.

However, developing an SI procedure for a new algorithm requires substantial manual effort and, until now, has been possible only for a restricted class of algorithms.
Conditioning on the selection requires an explicit description of the \emph{selection event}: the set of datasets that would have led the algorithm to the same selection, stated in a form a test can use.
Characterizing the selection event, and implementing the resulting inference, must be done anew for each algorithm, a non-trivial task that demands considerable statistical expertise.
Even then, existing methods succeed only when the selection event reduces to linear or, with further effort, quadratic inequalities in the data~\citep{lee2016exact,tibshirani2016exact,tsukurimichi2021conditional,duy2022more}.
Realistic algorithms readily exceed this limit: data-dependent comparisons involving products or ratios of intermediate quantities can induce polynomial inequalities of degree higher than two, which existing exact methods cannot handle.
Together, the manual effort and the restriction to linear or quadratic selection events have kept SI from being applied to many of the algorithms used in practice.

We introduce \emph{AutoSI}, a framework that derives the selection event automatically, so that the user only needs to write the algorithm as ordinary NumPy-like code.
The key observation is that the algorithm's output is the cumulative result of many small decisions made along the way (for example, which of two quantities is larger, which of several is the largest, or how a list of values is ordered).
Each such decision remains the same only over a certain range of the data.
For every data-dependent operation the algorithm performs (such as a comparison, an absolute value, a maximum, or a sort), AutoSI records the condition under which that operation's outcome would not change, accumulating these conditions as the algorithm runs.
Accumulated over one run, these conditions certify a range of data on which the output is provably unchanged; assembled across the data space, such ranges constitute exactly the selection event (the set of datasets on which the algorithm produces the same output).
The framework therefore obtains the selection event by running the algorithm itself, with no derivation by hand.

Beyond removing the manual effort, AutoSI also enlarges the class of algorithms to which SI applies, from those with linear or quadratic selection events to any algorithm expressible through rational functions of the data.
It represents every quantity an algorithm computes as a \emph{rational function} of the data (a ratio of two polynomials), so that the conditions it accumulates are polynomial inequalities, not just the linear or quadratic forms to which prior work is confined.
These inequalities are linear for the simplest algorithms but reach much higher degree for others.
A lasso tuned by a $3$-fold cross-validated $R^2$ score, for instance, yields inequalities of degree twelve (degree $4K$ for $K$ folds).
By solving these inequalities at any degree, AutoSI brings such procedures, which lie beyond the reach of every existing exact method, within the scope of valid SI.

Although AutoSI applies to a broad class of SI problems, we focus in this paper on feature selection in linear models, a canonical setting that has played a central role in the development of SI.
This setting provides a well-understood testbed for isolating AutoSI's two main advances: automatic construction of selection events and support for high-degree polynomial inequalities induced by rational computations.
We demonstrate these advances on three feature-selection methods of increasing complexity.
Marginal screening and the lasso lie within the scope of existing exact methods, whereas the lasso tuned by a cross-validated $R^2$ score does not.
The same automatic procedure handles all three.
Figure~\ref{fig:concept} provides an overview of how AutoSI turns ordinary array code into a selection event and, ultimately, a valid selective $p$-value.

\paragraph{Related work}
SI originated for feature selection in linear models, where the selection event is a set of linear inequalities in the data and inference is based on the resulting truncated normal distribution~\citep{lee2014exact,lee2016exact,tibshirani2016exact,fithian2014optimal,taylor2015statistical}, and has been extended to more complex selection algorithms~\citep{hyun2018exact,tsukurimichi2021conditional,le2021parametric,duy2022more}.
Cross-validation itself has been handled only in restricted forms: exactly when the validation score is a residual sum of squares~\citep{loftus2015cv}, or approximately by adding randomization~\citep{markovic2017unifying}.
Beyond feature selection, SI now covers clustering, change-point detection, and the outputs of deep neural networks~\citep{gao2022selective,jewell2022testing,duy2022quantifying,miwa2023valid,shiraishi2024statistical,miwa2024statistical}.
Most related to ours are recent frameworks that automate SI for composed procedures, though only at the granularity of predefined components with manually derived rules.
One framework supports feature-selection pipelines assembled from a fixed library of imputation, outlier-detection, and feature-selection modules~\citep{shiraishi2025pipelines}, and a related method addresses feature engineering~\citep{matsukawa2025afe}.
In contrast, AutoSI automates SI at the level of primitive operations, requiring no hand-derived rules and applying to any algorithm written as ordinary array code.
By representing quantities as rational functions of the data, it also handles the high-degree selection events that these component-level methods cannot.

\begin{figure*}[t]
\centering
\resizebox{\textwidth}{!}{%
\begin{tikzpicture}[
  font=\small,
  ptitle/.style={font=\small\bfseries},
  opname/.style={font=\scriptsize\ttfamily, anchor=east},
  flow/.style={-{Stealth[length=2.4mm]}, semithick, gray!60},
]
\colorlet{absC}{orange!85!black}
\colorlet{sortC}{violet!85!black}
\node[ptitle] at (2.35,3.6) {(a) Ordinary code};
\node[draw, rounded corners=2pt, inner sep=6pt, align=left, font=\scriptsize\ttfamily]
  (code) at (2.35,1.8) {%
def marginal\_screening(y):\\
\ \ corr = X.T @ y\\
\ \ idx = \textcolor{sortC}{asi.argmax}(\textcolor{absC}{asi.abs}(corr))\\
\ \ eta = (X.T @ X).inv() @ X.T\\
\ \ return eta[idx]};
\draw[flow] (5.1,1.1) -- (5.85,1.1)
  node[midway, above=2pt, font=\scriptsize, text=black] {run};
\begin{scope}[shift={(6.0,0)}]
  \node[ptitle] at (2.75,3.6) {(b) Calculating intervals along $\vY(z) = \va + \vb z$};
  \fill[gray!18] (1.9,0.35) rectangle (3.6,2.95);
  \draw[dashed, gray] (1.9,0.35) -- (1.9,2.95);
  \draw[dashed, gray] (3.6,0.35) -- (3.6,2.95);
  \node[opname, absC] at (1.4,2.6) {asi.abs};
  \draw[|-|, thick, absC] (1.5,2.6) -- (3.6,2.6);
  \node[opname, sortC] at (1.4,1.9) {asi.argmax};
  \draw[|-|, thick, sortC] (1.9,1.9) -- (5.3,1.9);
  \draw[->] (0.7,0.35) -- (5.5,0.35) node[below] {\scriptsize $z$};
  \draw[line width=2.2pt, blue!55!black] (1.9,0.35) -- (3.6,0.35);
  \node[font=\scriptsize] at (2.75,0.85) {intersection};
  \draw (2.95,0.42) -- (2.95,0.28);
  \node[font=\scriptsize] at (2.95,0.1) {$z_{\mathrm{obs}}$};
  \draw[blue!55!black, thin] (2.4,0.28) -- (2.4,-0.06);
  \node[font=\scriptsize, blue!55!black] at (2.4,-0.21) {$\gM = \{3\}$};
\end{scope}
\draw[flow] (8.8,-0.5) -- (7.3,-1.35)
  node[midway, right=4pt, font=\scriptsize, text=black] {sweep};
\begin{scope}[shift={(3.25,-5.2)}]
  \node[ptitle] at (2.5,3.6) {(c) Truncated normal on $\gZ \Rightarrow \pselective$};
  \fill[blue!15] plot[domain=0.6:1.3, samples=25]
    (\x,{0.35+2.25*exp(-0.5*((\x-2.5)/0.85)^2)}) -- (1.3,0.35) -- (0.6,0.35) -- cycle;
  \fill[blue!15] plot[domain=1.9:3.6, samples=40]
    (\x,{0.35+2.25*exp(-0.5*((\x-2.5)/0.85)^2)}) -- (3.6,0.35) -- (1.9,0.35) -- cycle;
  \fill[blue!15] plot[domain=4.1:4.7, samples=25]
    (\x,{0.35+2.25*exp(-0.5*((\x-2.5)/0.85)^2)}) -- (4.7,0.35) -- (4.1,0.35) -- cycle;
  \draw[semithick] plot[domain=0.2:4.8, samples=60]
    (\x,{0.35+2.25*exp(-0.5*((\x-2.5)/0.85)^2)});
  \draw[->] (0.1,0.35) -- (5.0,0.35) node[below] {\scriptsize $z$};
  \draw[line width=2.2pt, blue!55!black] (0.6,0.35) -- (1.3,0.35);
  \draw[line width=2.2pt, blue!55!black] (1.9,0.35) -- (3.6,0.35);
  \draw[line width=2.2pt, blue!55!black] (4.1,0.35) -- (4.7,0.35);
  \draw[blue!55!black, thin] (0.95,0.28) -- (2.18,-0.40);
  \draw[blue!55!black, thin] (2.3,0.28) -- (2.3,-0.40);
  \draw[blue!55!black, thin] (4.4,0.28) -- (2.42,-0.40);
  \node[font=\scriptsize, blue!55!black] at (2.3,-0.55) {$\gM = \{3\}$};
  \draw[gray, thin] (1.35,2.82) -- (1.6,0.5);
  \node[font=\scriptsize, gray] at (1.35,2.95) {$\gM = \{5\}$};
  \draw[gray, thin] (3.9,2.82) -- (3.85,0.5);
  \node[font=\scriptsize, gray] at (3.9,2.95) {$\gM = \{4\}$};
  \draw (2.95,0.42) -- (2.95,0.28);
  \node[font=\scriptsize] at (2.95,0.1) {$z_{\mathrm{obs}}$};
\end{scope}
\end{tikzpicture}%
}
\caption{%
Our framework, AutoSI, automatically derives the selection event of a user-written algorithm.
Given a short piece of ordinary code (a), AutoSI runs it along a one-dimensional line of datasets containing the observed one (the line is constructed in Section~\ref{sec:problem}) and, for every comparison and selection, records the interval of the line parameter $z$ on which that operation's outcome is unchanged (b); colors link each operation to its interval, and on the intersection the algorithm provably selects the same feature set $\gM$.
Sweeping the line and collecting the intervals whose selection matches the observed one yields the truncation region $\gZ$, from which a valid $p$-value (the selective $p$-value) follows via the truncated normal distribution (c).
The user writes only ordinary code.
}
\label{fig:concept}
\end{figure*}
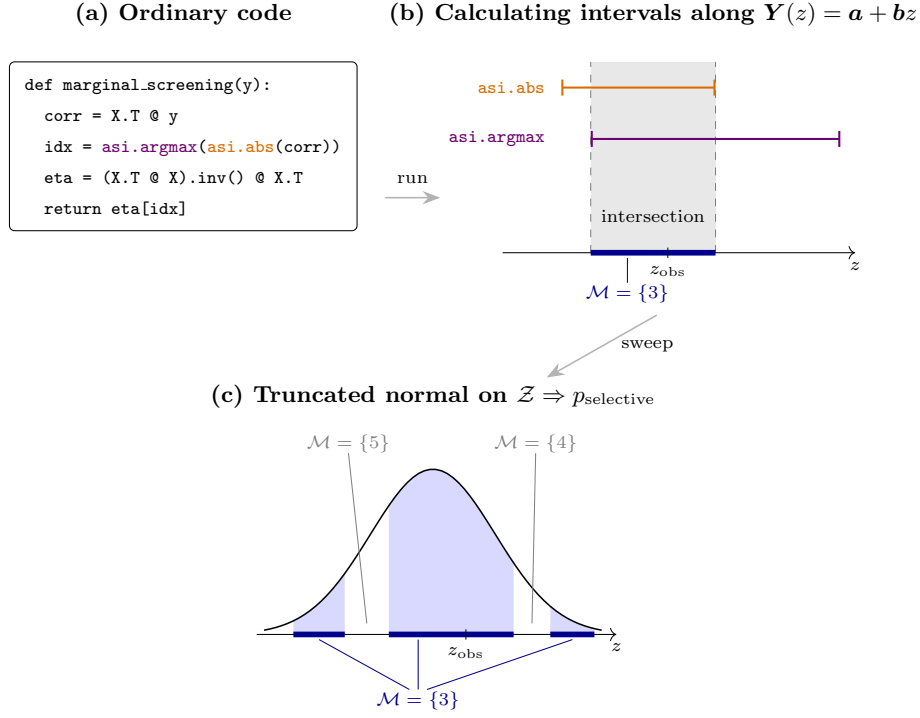

\paragraph{Contributions}
Our contributions are summarized as follows.\footnote{Our contribution is \emph{not} a new feature-selection method, but a framework that automates and generalizes SI; the three methods we study serve as proofs of concept.}
\begin{itemize}
  \item We introduce AutoSI, the first framework to derive an algorithm's selection event automatically at the level of primitive operations, eliminating the manual, per-algorithm derivation that SI has traditionally required.
  \item We generalize the class of tractable selection events from linear and quadratic to those induced by rational functions of the data. This enables valid SI for algorithms whose selection events are described by high-degree polynomial inequalities, such as model selection by a cross-validated $R^2$ score.
  \item We demonstrate AutoSI on three feature-selection methods implemented in a few dozen lines of code with no hand-derived analysis, showing on synthetic and real data that it controls the type I error rate with high power.
\end{itemize}

The implementation of the AutoSI library is available at \url{https://github.com/tkatsuoka/autosi}.

\newpage
\section{Problem Setting}
\label{sec:problem}

\subsection{Testing a selected feature}
\label{subsec:hypothesis}
We first state the general testing problem that SI addresses, and then specialize it to feature selection, the instance used throughout the paper.
In the general setting, an observed response $\vy \in \R^n$ is treated as a realization of
\begin{equation}
  \vY \sim \gN(\vmu, \mSigma),
  \label{eq:model}
\end{equation}
where $\vmu \in \R^n$ is an arbitrary unknown mean vector, with no structural restrictions imposed on it.
The covariance matrix $\mSigma \in \R^{n \times n}$ is assumed known, as is standard in exact SI~\citep{lee2016exact,fithian2014optimal}.

A selection algorithm extracts from the data a hypothesis to be tested, for example that a set of promising features is associated with the response.
Formally, an algorithm $\gA$ maps the observed response to a discrete output $\gA(\vy)$.
To assess this output statistically, we associate it with a scalar target---a single numerical quantity representing the aspect of the hypothesis to be tested.
Here the target is $\veta^\top \vmu$, a weighted combination of the entries of the unknown mean vector, where the weight vector $\veta \in \R^n$, called the \emph{contrast}, is determined by $\gA(\vy)$.
Specifically, we test
\begin{equation}
  \mathrm{H}_0 : \veta^\top \vmu = 0
  \quad \text{vs.} \quad
  \mathrm{H}_1 : \veta^\top \vmu \neq 0.
  \label{eq:hypothesis}
\end{equation}
Unlike conventional hypothesis testing, in which the hypothesis is fixed before the data are observed, the hypothesis here is data-dependent: the observed data $\vy$ determine the algorithm's output $\gA(\vy)$, which in turn determines $\veta$, and therefore which null hypothesis $\mathrm{H}_0 : \veta^\top \vmu = 0$ is tested.

Feature selection, our running instance, realizes this setting as follows.
Let $\mX \in \R^{n \times d}$ be a fixed design matrix whose columns are the $d$ candidate features, and let the algorithm return a selected set $\gM = \gA(\vy) \subseteq \{1, \dots, d\}$.
To assess a selected feature $j \in \gM$, we use its least-squares coefficient in a regression on all selected features.
The corresponding contrast is
\begin{equation}
  \veta^\top = \bm{e}_j^\top (\mX_{\gM}^\top \mX_{\gM})^{-1} \mX_{\gM}^\top,
\end{equation}
where $\mX_{\gM}$ contains the selected columns of $\mX$, and $\bm{e}_j$ selects the coordinate corresponding to feature $j$ within $\gM$.
With this definition, $\veta^\top \vy$ is the coefficient estimated from the noisy response $\vy$, whereas $\veta^\top \vmu$ is the coefficient obtained by replacing $\vy$ with its unknown mean $\vmu$.
Thus, $\mathrm{H}_0 : \veta^\top \vmu = 0$ asks whether feature $j$ has zero coefficient after accounting for the other selected features.
This target remains well defined even when $\vmu$ is not exactly linear in $\mX$.
Linear regression is used only to define the target, not as a model assumption required for validity.
Although feature selection is our running instance, AutoSI applies to any algorithm satisfying Condition~\ref{def:rational} and any output-dependent contrast $\veta$.

Testing this hypothesis as though it had been fixed in advance is invalid.
A natural test statistic is $\veta^\top \vY$, which is Gaussian for a fixed $\veta$, suggesting the naive $p$-value
\begin{equation}
  \pnaive = \PP_{\mathrm{H}_0}\!\left( |\veta^\top \vY| \geq |\veta^\top \vy| \right).
  \label{eq:naive}
\end{equation}
This quantity treats $\veta$ as fixed and ignores that it was chosen using $\vy$; because of this double dipping, $\pnaive$ does not control the type I error rate.

\subsection{Selective inference}
\label{subsec:si}
SI restores validity by conditioning on the selection~\citep{lee2016exact,fithian2014optimal,taylor2015statistical}.
It evaluates the test statistic under the distribution of $\vY$ conditional on the \emph{selection event} $\{\gA(\vY) = \gA(\vy)\}$ that the algorithm returns the observed selection.
\paragraph{Reduction to a line search problem}
To make this conditional distribution computable, SI additionally conditions on the sufficient statistic of the nuisance parameter, $\gQ(\vY) = (\mI_n - \vb \veta^\top)\, \vY$ with $\vb = \mSigma \veta / (\veta^\top \mSigma \veta) \in \R^n$.
Intuitively, this is the part of $\vY$ that carries no information about the target parameter $\veta^\top \vmu$~\citep{fithian2014optimal,lee2016exact}.
Fixing $\gQ(\vY)$ at its observed value leaves a single scalar degree of freedom, the test statistic itself.
The response $\vY$ is then confined to the one-dimensional line
\begin{equation}
  \vY(z) = \va + \vb z, \quad z \in \R,
  \label{eq:line}
\end{equation}
with offset $\va = \gQ(\vy) \in \R^n$, along which $\veta^\top \vY(z) = z$ is distributed under $\mathrm{H}_0$ as $\gN(0,\, \veta^\top \mSigma \veta)$.
This additional conditioning eliminates every unknown except the target parameter $\veta^\top \vmu$ and does not compromise the validity of the test.
We take the reduction as given, and Lemma~\ref{lem:line} in Appendix~\ref{app:proof} makes it precise.

\paragraph{The selective $p$-value}
With $\vY$ confined to the line, computing the selective $p$-value reduces to identifying the values of $z$ that reproduce the observed selection.
Writing $z = \veta^\top \vY$ for the test statistic and $z_{\mathrm{obs}} = \veta^\top \vy$ for its observed value, the selective $p$-value is defined as
\begin{equation}
  \pselective = \PP_{\mathrm{H}_0}\!\left( |z| \geq |z_{\mathrm{obs}}| \;\middle|\; z \in \gZ \right),
  \label{eq:selective}
\end{equation}
where $\gZ$ is the truncation region defined as
\begin{equation}
  \gZ = \left\{ z \in \R \;\middle|\; \gA(\va + \vb z) = \gA(\vy) \right\},
  \label{eq:truncation}
\end{equation}
which collects the points on the line that yield the same hypothesis selection.
Once $\gZ$ is known, $\pselective$ follows in closed form from the truncated normal distribution (the Gaussian law of $z$ restricted to $\gZ$) and controls the type I error rate exactly, in the following standard sense.
\begin{theorem}[Validity of the selective $p$-value; e.g., \citealp{lee2016exact,fithian2014optimal}]
\label{thm:validity}
The selective $p$-value of Eq.~\eqref{eq:selective} satisfies
\begin{equation}
  \PP_{\mathrm{H}_0}\!\left( \pselective \leq \alpha \right) = \alpha
  \quad \text{for all } \alpha \in [0, 1].
  \label{eq:validity}
\end{equation}
\end{theorem}
The entire difficulty is thus concentrated in a single step: determining $\gZ$.
Existing SI methods require an expert to derive, separately for each algorithm, the set $\gZ$ of values of $z$ that reproduce the observed selection.
AutoSI automates this algorithm-specific derivation.

\newpage
\section{Method}
\label{sec:method}
The proposed method, implemented in a library named AutoSI, computes the truncation region $\gZ$ automatically.
It runs the user's algorithm along the line $\vY(z) = \va + \vb z$ and records, for every data-dependent operation, the range of $z$ over which that operation's outcome is unchanged.
The starting point is a change of representation.
Every array the algorithm manipulates is carried not as a single numeric vector but as an array of \emph{rational functions of $z$}.
Each intermediate quantity is thereby available as a function of $z$, not merely as its value at the observed point (Section~\ref{subsec:representation}).
This representation turns every data-dependent decision the algorithm makes (a comparison or a selection) into a sign condition on a rational function of $z$.
Solving the sign conditions identifies the interval of $z$ on which every decision, and hence the output, stays unchanged (Sections~\ref{subsec:events}--\ref{subsec:rootfinding}).
Sweeping the line interval by interval assembles the whole of $\gZ$ (Section~\ref{subsec:accumulation}).
Section~\ref{subsec:interface} then shows the single call through which the user runs the entire procedure.

\subsection{Data as rational functions of $z$}
\label{subsec:representation}
Each element of every array is represented as a rational function of the line parameter $z$,
\begin{equation}
  \text{value}(z) = \frac{\num(z)}{\den(z)},
  \label{eq:rational}
\end{equation}
stored as the coefficient vectors of the two polynomials $\num(z)$ and $\den(z)$.
The response itself is the affine function $\vY(z) = \va + \vb z$ of Eq.~\eqref{eq:line}, in which each element has a degree-one numerator and a constant denominator.
The four arithmetic operations on these arrays stay within the representation, because sums, differences, products, and quotients of rational functions are again rational.
A matrix product stays within it as well, because each of its entries is a sum of products of rational functions.
The method executes each such operation as polynomial arithmetic on the coefficient vectors, imposing no constraint on $z$.
Each operation thus updates the stored coefficient vectors, and their degrees grow along the way: multiplication, for example, adds the degrees of its operands (Appendix~\ref{app:operations} details this degree growth).

\subsection{Selection events as polynomial inequalities}
\label{subsec:events}
Unlike the arithmetic, a comparison or a selection produces a discrete outcome, and that outcome depends on where $z$ lies.
These operations, the algorithm's data-dependent branches, are what carve the line into pieces.
A comparison between two computed quantities, $g(z) > h(z)$, holds exactly when the rational function $g(z) - h(z) = \num(z)/\den(z)$ is positive.
Each comparison thus reduces to a sign condition on a single rational function.
Selections reduce to conjunctions of such comparisons.
An $\argmax$ that selects element $k$ from $\{x_i(z)\}$, for instance, imposes $x_k(z) - x_i(z) > 0$ for every competing $i \neq k$.
Absolute-value and sorting operations decompose in the same way (Appendix~\ref{app:operations}).

One run of the algorithm therefore records finitely many such sign conditions, which we write as $\num_\ell(z)/\den_\ell(z) > 0$ for $\ell = 1, \dots, L$.
The sign of each $\num_\ell$ is chosen so that the condition holds at the point of the run.
For the run at the observed $z_{\mathrm{obs}}$, together they certify
\begin{equation}
  \bigcap_{\ell=1}^{L} \left\{ z \in \R \;\middle|\; \frac{\num_\ell(z)}{\den_\ell(z)} > 0 \right\}
  \subseteq \gZ,
  \label{eq:conjunction}
\end{equation}
that is, a system of polynomial inequalities in $z$.
The inclusion in the truncation region $\gZ$ of Eq.~\eqref{eq:truncation} holds because identical branch outcomes force identical output.
The whole of $\gZ$ is recovered exactly as a union of such regions in Section~\ref{subsec:accumulation}.
The degree of these polynomial inequalities is determined by the algorithm's complexity: it is one for the linear comparisons of marginal screening, but it reaches $4K$ for the lasso tuned by a $K$-fold cross-validated $R^2$ score (Section~\ref{subsec:examples}).
We now state the condition that determines the class of algorithms to which AutoSI applies.
Informally, an algorithm in this class must express all numerical computations through rational arithmetic and make every data-dependent decision by comparing the quantities it computes.
We refer to such algorithms as \emph{rationally expressible}, as formalized below.
\begin{condition}[Rationally expressible algorithm]
\label{def:rational}
An algorithm is rationally expressible if
(i) it is deterministic;
(ii) it performs numerical computations only through the arithmetic operations of Section~\ref{subsec:representation};
(iii) it makes every data-dependent decision---including branching, element selection, and loop termination---through the comparisons and selections described above; and
(iv) it terminates for every input.
\end{condition}
The condition unifies algorithms whose selection events have so far been characterized separately, typically one algorithm and one work at a time (Table~\ref{tab:coverage}).
Its rows differ only in the degree of the polynomial inequalities their operations generate.
The last row (CV by $R^2$ ($K$-fold)), whose degree exceeds two, is handled by no existing method.

\begin{table}[t]
\centering
\small
\begin{tabular}{@{}lcl@{}}
\toprule
Algorithm & Degree & Hand-derived SI \\
\midrule
Marginal screening$^\dagger$      & $1$     & \citet{lee2014exact} \\
Forward stepwise selection        & $1$--$2$ & \citet{tibshirani2016exact} \\
Lasso (fixed $\lambda$)$^\dagger$ & $1$     & \citet{lee2016exact} \\
$k$-means clustering              & $2$     & \citet{chen2023kmeans} \\
CV by squared error               & $2$     & \citet{loftus2015cv} \\
CV by $R^2$ ($K$-fold)$^\dagger$  & $4K$    & --- \\
\bottomrule
\end{tabular}
\caption{%
Selection algorithms covered by Condition~\ref{def:rational}, the degree of the polynomial inequalities constituting their selection events, and the work that previously derived each event by hand ($\dagger$: demonstrated in Section~\ref{sec:experiments}).
Our experiments use $K = 3$ folds, so the last row has degree twelve.
}
\label{tab:coverage}
\end{table}

\subsection{Solving the inequalities by root-finding}
\label{subsec:rootfinding}
The set of $z$ satisfying a single sign condition follows from the roots of its two polynomials.
Consider a run of the algorithm at any point $z \in \R$ and one of its recorded conditions, $\num(z)/\den(z) > 0$.
The condition holds at the run point by the choice of sign in Section~\ref{subsec:events}.
Because a polynomial can change sign only at its real roots, the condition keeps holding until the first root of the numerator or the denominator on either side of $z$.
All conditions from the run therefore hold simultaneously up to the first recorded root on either side of $z$.
Let $\mathcal{R}(z)$ denote the set of real roots of the polynomials $\num_\ell$ and $\den_\ell$ recorded during the run at $z$.
The run then certifies the open interval
\begin{equation}
  \big( I^{-}(z),\, I^{+}(z) \big)
  := \left( \max_{\substack{r \in \mathcal{R}(z) \\ r < z}} r, \;\;\, \min_{\substack{r \in \mathcal{R}(z) \\ r > z}} r \right),
  \label{eq:isosign}
\end{equation}
where the maximum (minimum) is taken as $-\infty$ ($+\infty$) when no such root exists.

The method computes the real roots numerically as the real eigenvalues of the companion matrix.
Root-finding works at any degree, so the degree-$4K$ inequalities of a cross-validated $R^2$ score are no harder to characterize than the linear ones of marginal screening.
The restriction to linear or quadratic cases disappears.

\subsection{Assembling the truncation region}
\label{subsec:accumulation}
\begin{algorithm}[t]
\caption{Automatic computation of the selective $p$-value}
\label{alg:autosi}
\begin{algorithmic}[1]
\REQUIRE $\vy$, $\gA$, $\veta$, $\mSigma$ (Section~\ref{sec:problem})
\STATE compute $\va$ and $\vb$ of Eq.~\eqref{eq:line}
\STATE $\gZ \leftarrow \emptyset$; \quad initialize $z$ to a sufficiently small value
\WHILE{$z$ is not large enough}
\STATE run $\gA$ on $\vY(z) = \va + \vb z$ as rational functions of $z$, deciding each branch at the current $z$
\STATE compute $\big( I^{-}(z),\, I^{+}(z) \big)$ by Eq.~\eqref{eq:isosign}
\IF{$\gA(\va + \vb z) = \gA(\vy)$}
\STATE $\gZ \leftarrow \gZ \cup \big( I^{-}(z),\, I^{+}(z) \big)$
\ENDIF
\STATE $z \leftarrow I^{+}(z) + \gamma$, where $0 < \gamma \ll 1$
\ENDWHILE
\RETURN $\pselective$ of Eq.~\eqref{eq:selective} over the truncation region $\gZ$
\end{algorithmic}
\end{algorithm}
Running the algorithm once at $z_{\mathrm{obs}}$ certifies the interval $\big( I^{-}(z_{\mathrm{obs}}),\, I^{+}(z_{\mathrm{obs}}) \big)$ of Eq.~\eqref{eq:isosign}, which contains $z_{\mathrm{obs}}$ and on which the algorithm returns the observed output.
A single run, however, certifies only the interval containing the point at which it was run.
The method therefore restarts the algorithm just past its right endpoint and repeats, sweeping the line interval by interval.
The sweep thereby recovers the truncation region as
\begin{equation}
  \gZ
  = \bigcup_{t \,:\, \gA(\va + \vb z_t) = \gA(\vy)} \big( I^{-}(z_t),\, I^{+}(z_t) \big),
  \label{eq:union}
\end{equation}
where $z_1 < z_2 < \cdots$ are the visited points.
For any rationally expressible algorithm, Eq.~\eqref{eq:union} holds up to a finite set of points, provided that the sweep uses a sufficiently small step $\gamma$.
This exploration of all intervals along the line can be conducted by a technique called \emph{parametric programming} in the SI literature~\citep{le2021parametric,duy2022more}, originally developed for convex optimization.
The selective $p$-value of Eq.~\eqref{eq:selective} is then computed from the truncated normal distribution over $\gZ$, and Algorithm~\ref{alg:autosi} collects the whole procedure.

Although each single run conditions on the entire computation path, the union keeps every $z$ whose \emph{output} coincides with the observed one, whether or not the path does.
The final inference therefore conditions on exactly the selection event and inherits none of the conservativeness of over-conditioning.

\subsection{Using AutoSI}
\label{subsec:interface}
AutoSI provides a tracked array type with an interface that mirrors NumPy operations.
The tracking described in Sections~\ref{subsec:representation}--\ref{subsec:events} is handled automatically, so the user can write the selection algorithm as ordinary NumPy-like code.
After wrapping the observed data in tracked arrays, the user runs the algorithm once on the observed response to obtain the selected contrast $\veta$.
The user then calls the inference routine with the algorithm, $\veta$, and the noise variance (Listing~\ref{lst:ms}).
Internally, the routine reruns the algorithm along the line $\vY(z) = \va + \vb z$, constructs the truncation region $\gZ$ using Algorithm~\ref{alg:autosi}, and returns the selective $p$-value.

\begin{listing}[t]
\begin{lstlisting}
import autosi as asi

# wrap the observed NumPy data
y = asi.array(y)
X = asi.array(X)
idx = 0  # which feature to test

def marginal_screening(y):
    corr = X.T @ y
    order = asi.argsort(asi.abs(corr))
    top3 = asi.sort(order[-3:])
    X_M = X.T[top3].T
    eta = (X_M.T @ X_M).inv() @ X_M.T
    return eta[idx]

eta = marginal_screening(y)

result = asi.inference(
    eta=eta, prob_vec=y, var=1.0,
    algorithm=marginal_screening)
print(result.p_value)
\end{lstlisting}
\caption{%
The complete SI for marginal screening.
}
\label{lst:ms}
\end{listing}

\newpage
\section{Experiments}
\label{sec:experiments}

\subsection{Three feature-selection methods}
\label{subsec:examples}
%
\begin{figure*}[t]
\centering
\includegraphics[height=0.06\textwidth]{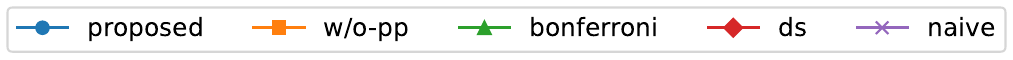}\\[3pt]
\begin{subfigure}{0.48\textwidth}
  \centering
  \includegraphics[width=\linewidth]{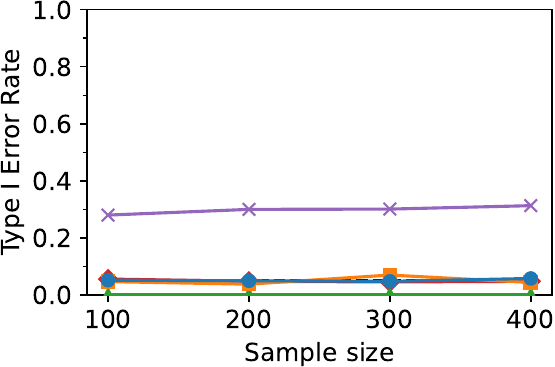}
  \caption{Marginal screening (independence)}
\end{subfigure}\hfill%
\begin{subfigure}{0.48\textwidth}
  \centering
  \includegraphics[width=\linewidth]{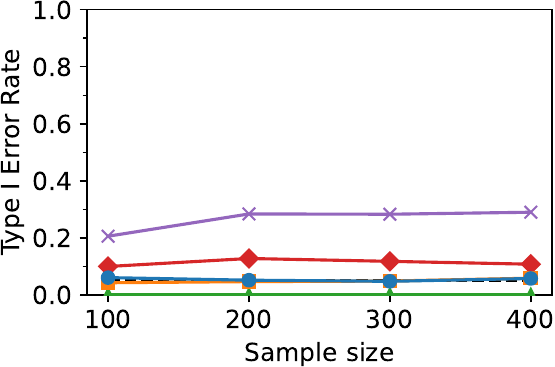}
  \caption{Marginal screening (correlation)}
\end{subfigure}\\[4pt]
\begin{subfigure}{0.48\textwidth}
  \centering
  \includegraphics[width=\linewidth]{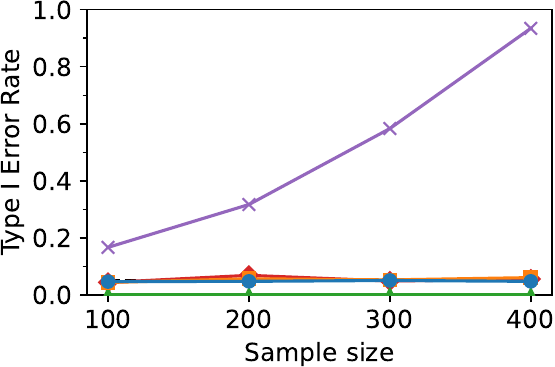}
  \caption{Lasso (independence)}
\end{subfigure}\hfill%
\begin{subfigure}{0.48\textwidth}
  \centering
  \includegraphics[width=\linewidth]{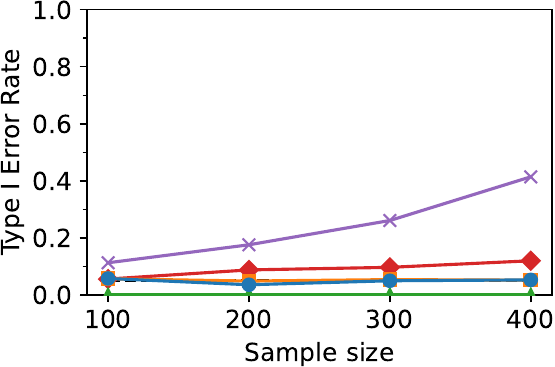}
  \caption{Lasso (correlation)}
\end{subfigure}\\[4pt]
\begin{subfigure}{0.48\textwidth}
  \centering
  \includegraphics[width=\linewidth]{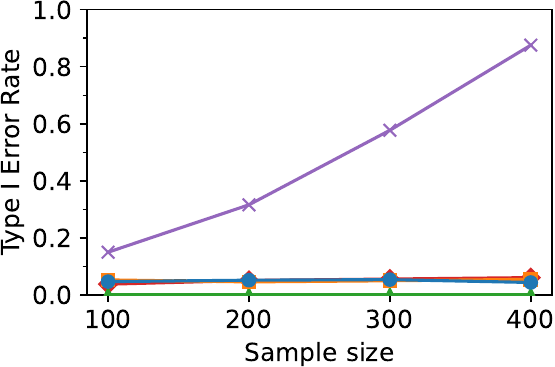}
  \caption{Cross-validated lasso (independence)}
\end{subfigure}\hfill%
\begin{subfigure}{0.48\textwidth}
  \centering
  \includegraphics[width=\linewidth]{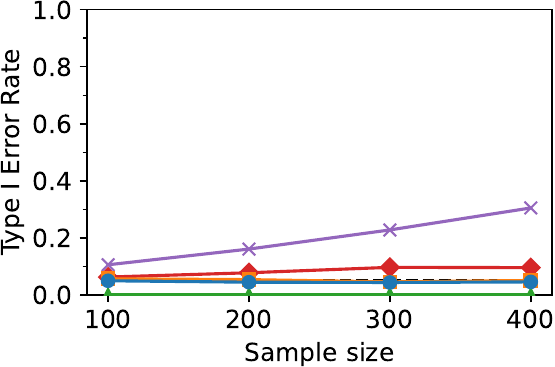}
  \caption{Cross-validated lasso (correlation)}
\end{subfigure}
\caption{%
Type I error rates on synthetic data as the sample size $n$ grows.
\texttt{proposed} and \texttt{w/o-pp} control the type I error rate at $\alpha = 0.05$ (dashed line), \texttt{ds} only under independence, \texttt{bonferroni} far below the level, and \texttt{naive} far exceeds it.
}
\label{fig:fpr}
\end{figure*}
\begin{figure*}[t]
\centering
\includegraphics[height=0.06\textwidth]{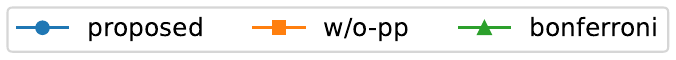}\\[3pt]
\begin{subfigure}{0.48\textwidth}
  \centering
  \includegraphics[width=\linewidth]{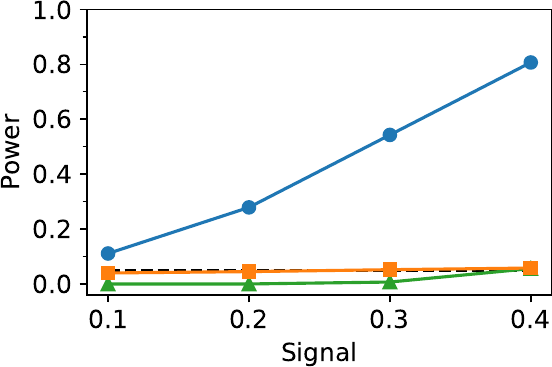}
  \caption{Marginal screening (independence)}
\end{subfigure}\hfill%
\begin{subfigure}{0.48\textwidth}
  \centering
  \includegraphics[width=\linewidth]{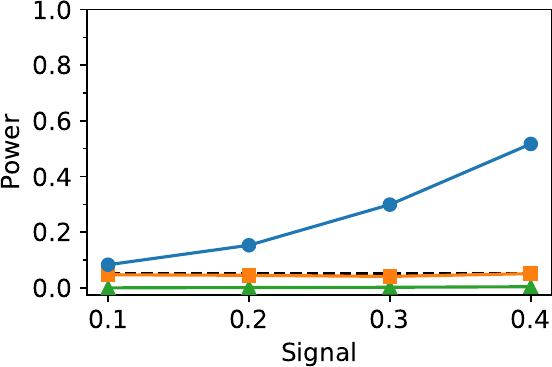}
  \caption{Marginal screening (correlation)}
\end{subfigure}\\[4pt]
\begin{subfigure}{0.48\textwidth}
  \centering
  \includegraphics[width=\linewidth]{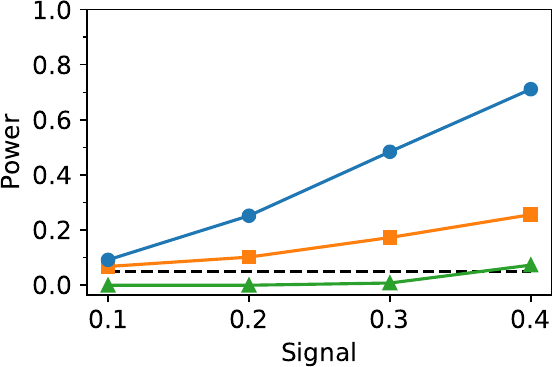}
  \caption{Lasso (independence)}
\end{subfigure}\hfill%
\begin{subfigure}{0.48\textwidth}
  \centering
  \includegraphics[width=\linewidth]{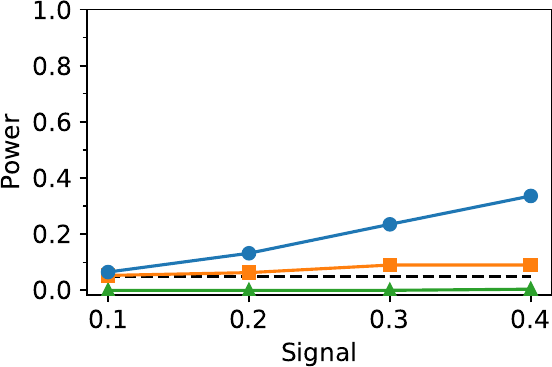}
  \caption{Lasso (correlation)}
\end{subfigure}\\[4pt]
\begin{subfigure}{0.48\textwidth}
  \centering
  \includegraphics[width=\linewidth]{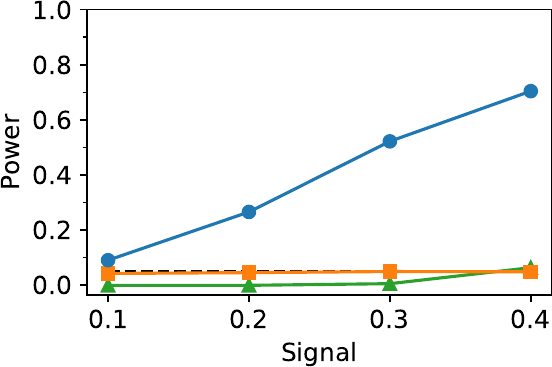}
  \caption{Cross-validated lasso (independence)}
\end{subfigure}\hfill%
\begin{subfigure}{0.48\textwidth}
  \centering
  \includegraphics[width=\linewidth]{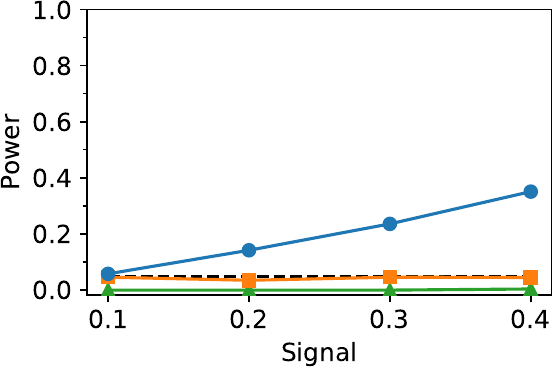}
  \caption{Cross-validated lasso (correlation)}
\end{subfigure}
\caption{%
Power on synthetic data as the signal strength $\Delta$ grows; \texttt{naive} and \texttt{ds} are excluded because they do not control the type I error rate.
Among the methods that control the type I error rate, \texttt{proposed} attains the highest power in every setting.
}
\label{fig:power}
\end{figure*}
We demonstrate the framework on three feature-selection methods whose selection events range from linear to degree twelve.
Marginal screening (Listing~\ref{lst:ms}) compares correlations, and the lasso adds an iterative fit whose updates involve only comparisons and linear arithmetic, so both selection events stay at degree one.
The cross-validated lasso instead selects its regularization strength by comparing $K$-fold $R^2$ scores across candidate values, and because each score is a sum of ratios, the comparisons yield polynomial inequalities of degree $4K$, as derived in Appendix~\ref{app:operations}.
The root-finding of Section~\ref{subsec:rootfinding} resolves them with no change to the user's code (see Appendix~\ref{app:cvlisting}).

\subsection{Experimental setup}
\label{subsec:setup}
Each method selects a feature set (the $k = 3$ most correlated ones for marginal screening, those with nonzero coefficients for the lasso variants), and we compare five ways of testing one feature drawn from the set uniformly at random, independently of the data (precise definitions in Appendix~\ref{app:comparison}).
\begin{itemize}
  \item \texttt{proposed}: the selective $p$-value of Eq.~\eqref{eq:selective} computed by Algorithm~\ref{alg:autosi}.
  \item \texttt{w/o-pp}: an ablation of \texttt{proposed} without parametric programming, taking $\gZ = ( I^{-}(z_{\mathrm{obs}}),\, I^{+}(z_{\mathrm{obs}}) )$.
  \item \texttt{naive}: the unadjusted $p$-value of Eq.~\eqref{eq:naive}.
  \item \texttt{bonferroni}: the naive $p$-value with a Bonferroni correction for multiple testing.
  \item \texttt{ds}: data splitting, which selects on one half of the data and tests on the other half with a $z$-test.
\end{itemize}
Synthetic data are generated from the linear model $\vy = \mX \vbeta + \veps$ with $d = 20$ features and noise $\veps \sim \gN(\bm{0}, \mSigma)$, so that $\vy$ follows Eq.~\eqref{eq:model} with $\vmu = \mX \vbeta$.
In the independence setting, we set $\mSigma = \mI_n$ and draw the entries of $\mX \in \R^{n \times d}$ as independent standard normals.
In the correlation setting, we set $\mSigma = \big(0.5^{|i-j|}\big)_{ij}$ and draw each column of $\mX$ from $\gN(\bm{0}, \mSigma)$, so that both the noise and the covariates are dependent across observations.
In the null setting all coefficients of $\vbeta$ are zero, whereas in the alternative setting the first three coefficients are set to a signal strength $\Delta > 0$ and the rest to zero.
%
For the type I error rate we collect $1{,}000$ null $p$-values for each $n \in \{100, 200, 300, 400\}$.
For the power we fix $n = 100$ and vary $\Delta \in \{0.1, 0.2, 0.3, 0.4\}$; because a rejection demonstrates power only when a truly active feature is tested, we report the rejection rate over $1{,}000$ trials in which the tested feature is an active one.
All tests are two-sided at significance level $\alpha = 0.05$, and the hyperparameters of the three methods (the regularization strength, the stopping tolerance, and the cross-validation grid) are listed in Appendix~\ref{app:expdetails} together with the computing environment.

\subsection{Results on synthetic data}
\label{subsec:synthetic}
Figure~\ref{fig:fpr} compares the type I error rates.
Across all three feature-selection methods and both settings, \texttt{proposed} and \texttt{w/o-pp} control the type I error rate at the significance level.
The \texttt{ds} baseline controls the type I error rate under independence, but fails to do so under correlation, because the two halves of the data are dependent (Appendix~\ref{app:comparison}).
As with \texttt{naive}, we therefore exclude \texttt{ds} from the power comparison.
The \texttt{bonferroni} correction controls the type I error rate far below the level, reflecting its conservativeness.
In contrast, \texttt{naive} exceeds the significance level by a wide margin, so its detections are not statistically reliable.
Figure~\ref{fig:power} compares the powers of the methods that control the type I error rate.
Among them, \texttt{proposed} attains the highest power in every setting.
The gap between \texttt{proposed} and \texttt{w/o-pp} reflects the power cost of over-conditioning (Section~\ref{subsec:accumulation}).

\subsection{Results on real data}
\label{subsec:realdata}
We further evaluate the framework on seven regression datasets from the UCI Machine Learning Repository~\citep{kelly2023uci}.
For each dataset we repeatedly subsample $n = 100$ observations, standardize the features and the response, and estimate the noise variance from the least-squares residuals of the subsample; the full protocol is given in Appendix~\ref{app:expdetails}.
Because the true signal in real data is unknown, the type I error rate cannot be measured directly; we therefore report the rejection rate at $\alpha = 0.05$.
We omit \texttt{naive} and \texttt{ds}, which do not control the type I error rate (Section~\ref{subsec:synthetic}).
For the remaining methods, the rejection rate measures their ability to detect real associations.
Table~\ref{tab:realdata} shows the results for the cross-validated lasso, and Table~\ref{tab:realdata-all} in Appendix~\ref{app:expdetails} reports all three feature-selection methods.
For the cross-validated lasso, \texttt{proposed} attains the highest rejection rate on every dataset.

\begin{table}[t]
\centering
\small
\setlength{\tabcolsep}{4pt}
\begin{tabular}{@{}lccc@{}}
\toprule
Dataset & \texttt{bonferroni} & \texttt{w/o-pp} & \texttt{proposed} \\
\midrule
Airfoil          & .65 & .04 & \textbf{.84} \\
Concrete         & .48 & .07 & \textbf{.56} \\
Energy (heating) & .64 & .08 & \textbf{.70} \\
Energy (cooling) & .68 & .05 & \textbf{.71} \\
Real estate      & .30 & .02 & \textbf{.50} \\
Wine (red)       & .13 & .08 & \textbf{.34} \\
Wine (white)     & .16 & .05 & \textbf{.30} \\
\bottomrule
\end{tabular}
\caption{%
Rejection rates at $\alpha = 0.05$ on seven real datasets for the cross-validated lasso.
The highest rate is shown in bold (ties broken toward \texttt{proposed}).
}
\label{tab:realdata}
\end{table}

\newpage
\section{Discussion and Conclusion}
\label{sec:conclusion}
We presented AutoSI, a framework that traces the primitive array operations of a user-written algorithm and automatically constructs the corresponding selection event.
By representing every intermediate quantity as a rational function of the data, AutoSI extends exact SI beyond linear and quadratic selection events to those of arbitrary polynomial degree.
Consequently, users need only implement a rationally expressible algorithm in ordinary NumPy-like code.
AutoSI derives its selection event and computes a valid selective $p$-value automatically.
Using the same framework, we handled marginal screening, the lasso, and a cross-validated lasso beyond the reach of existing exact methods, while controlling the type I error rate and achieving higher power than the valid baselines.

The framework applies to the rationally expressible algorithms of Condition~\ref{def:rational}, and its main cost is computation time, which grows with the selection event's degree.

The next step is to carry the framework to other previously unreachable selection algorithms.
Removing the manual derivation will let practitioners attach valid $p$-values to the algorithms they use, rather than to simplified variants.

\newpage
\subsection*{Acknowledgments}
This work was supported by RIKEN Junior Research Associate Program and partially supported by JST CREST (JPMJCR21D3, JPMJCR22N2), JST Moonshot R\&D (JPMJMS2033-05), and RIKEN Center for Advanced Intelligence Project.

\clearpage
\appendix
\newpage

\section{Proofs}
\label{app:proof}
Throughout, fix the line $\vY(z) = \va + \vb z$ of Eq.~\eqref{eq:line} and write $\gA(z)$ as shorthand for the selection $\gA(\va + \vb z)$.

\subsection{Proof of Theorem~\ref{thm:validity}}
\label{app:proof-validity}
The theorem follows from the classical reduction to a one-dimensional line stated in Section~\ref{sec:problem}, which we include for completeness.

\begin{lemma}[Reduction to a one-dimensional line; e.g., \citealp{lee2016exact,fithian2014optimal}]
\label{lem:line}
Conditional on the event $\{\gA(\vY) = \gA(\vy)\} \cap \{\gQ(\vY) = \gQ(\vy)\}$, the test statistic $\veta^\top \vY$ follows the normal distribution $\gN(\veta^\top \vmu,\, \veta^\top \mSigma \veta)$ truncated to the region $\gZ$ of Eq.~\eqref{eq:truncation}.
\end{lemma}
\begin{proof}
Since $\veta^\top \vb = 1$, the map $\vY \mapsto (\veta^\top \vY,\, \gQ(\vY))$ is an invertible affine transformation with inverse $\vY = \gQ(\vY) + \vb\, (\veta^\top \vY)$.
Its components are jointly Gaussian and uncorrelated,
\begin{equation}
\begin{split}
  \mathrm{cov}\!\left( \gQ(\vY),\, \veta^\top \vY \right)
  &= (\mI_n - \vb \veta^\top)\, \mSigma \veta \\
  &= \mSigma \veta - \vb\, (\veta^\top \mSigma \veta)
  = \bm{0},
\end{split}
\end{equation}
hence independent, and $\veta^\top \vY \sim \gN(\veta^\top \vmu,\, \veta^\top \mSigma \veta)$.
Writing $z = \veta^\top \vY$, the conditioning event becomes
\begin{equation}
\begin{split}
  &\big\{ \vY \bigm| \gA(\vY) = \gA(\vy),\, \gQ(\vY) = \gQ(\vy) \big\} \\
  &\quad= \big\{ \va + \vb z \bigm| \gA(\va + \vb z) = \gA(\vy),\, z \in \R \big\} \\
  &\quad= \big\{ \va + \vb z \bigm| z \in \gZ \big\},
\end{split}
\end{equation}
where the first equality holds because $\gQ(\vY) = \gQ(\vy)$ forces $\vY = \gQ(\vy) + \vb z = \va + \vb z$, and the second is the definition of $\gZ$ in Eq.~\eqref{eq:truncation}.
By independence, conditioning on $\gQ(\vY)$ leaves the law of $z$ unchanged, and the restriction to $\{z \in \gZ\}$ truncates it.
\end{proof}

\begin{proof}[Proof of Theorem~\ref{thm:validity}]
By Lemma~\ref{lem:line} and the probability integral transform, under $\mathrm{H}_0$,
\begin{equation}
  \PP_{\mathrm{H}_0}\!\left( \pselective \leq \alpha \;\middle|\; \gA(\vY) = \gA(\vy),\, \gQ(\vY) = \gQ(\vy) \right)
  = \alpha
\end{equation}
for every $\alpha \in [0, 1]$ and every realization $\vy$.
Marginalizing over $\gQ(\vY)$ gives
\begin{equation}
\begin{split}
  &\PP_{\mathrm{H}_0}\!\left( \pselective \leq \alpha \mid \gA(\vY) = \gA(\vy) \right) \\
  &= \int_{\R^n}
     \PP_{\mathrm{H}_0}\!\left( \pselective \leq \alpha \mid \gA(\vY) = \gA(\vy),\, \gQ(\vY) = \bm{q} \right) \\
  &\hspace{3.4em}
     \PP_{\mathrm{H}_0}\!\left( \gQ(\vY) = \bm{q} \mid \gA(\vY) = \gA(\vy) \right) d\bm{q}
   = \alpha,
\end{split}
\end{equation}
and marginalizing over the selection gives
\begin{equation}
\begin{split}
  &\PP_{\mathrm{H}_0}\!\left( \pselective \leq \alpha \right) \\
  &= \sum_{\gA(\vy)}
     \PP_{\mathrm{H}_0}\!\left( \gA(\vy) \right)
     \PP_{\mathrm{H}_0}\!\left( \pselective \leq \alpha \mid \gA(\vY) = \gA(\vy) \right) \\
  &= \sum_{\gA(\vy)}
     \PP_{\mathrm{H}_0}\!\left( \gA(\vy) \right) \alpha
   = \alpha,
\end{split}
\end{equation}
where the sum ranges over all possible values of $\gA(\vy)$ and $\PP_{\mathrm{H}_0}(\gA(\vy))$ is shorthand for $\PP_{\mathrm{H}_0}(\gA(\vY) = \gA(\vy))$; this is Eq.~\eqref{eq:validity}.
\end{proof}
The name of the conditioning variable reflects the decomposition $\vmu = (\veta^\top \vmu)\, \vb + (\mI_n - \vb \veta^\top)\, \vmu$, whose second term is the nuisance parameter: everything about $\vmu$ that the hypotheses of Eq.~\eqref{eq:hypothesis} do not constrain.
The law of $\gQ(\vY)$ depends only on this nuisance component, whereas by Lemma~\ref{lem:line} the conditional law of the test statistic is free of it; conditioning on sufficient statistics of nuisance parameters is the standard device for obtaining such nuisance-free conditional tests~\citep{fithian2014optimal}.

\section{Per-Operation Selection Events}
\label{app:operations}
This appendix records the sign conditions each tracked operation of Table~\ref{tab:ops} imposes; the real roots of their numerators and denominators constitute the root set $\mathcal{R}(z)$ of Eq.~\eqref{eq:isosign} for a run at $z$.

\begin{table}[t]
\centering
\small
\begin{tabular}{@{}llc@{}}
\toprule
Group & Examples & Tracked? \\
\midrule
Arithmetic & $+\;\,-\;\,\times\;\,/$,\; \texttt{@},\; \texttt{inv} & --- \\
Comparison & $<\;\,>\;\,\leq\;\,\geq$ & \checkmark \\
Selection & \texttt{abs},\, \texttt{max},\, \texttt{argmax},\, \texttt{argsort} & \checkmark \\
Equality & \texttt{==},\; \texttt{!=} & --- \\
\bottomrule
\end{tabular}
\caption{%
Operations supported by the array type: each comparison and selection contributes sign conditions (``\checkmark''), whereas arithmetic and equality tests are untracked (``---'').
}
\label{tab:ops}
\end{table}

\paragraph{Comparisons}
$g(z) > h(z)$ holds exactly when $g - h = \num/\den$ is positive, that is, when $\num$ and $\den$ share their sign; the condition contributes the real roots of both polynomials to $\mathcal{R}(z)$.
The operators $<$, $\leq$, and $\geq$ are analogous; ties occur only at roots, a null set.

\paragraph{Absolute value}
$\abs{x}$ returns $x$ or $-x$ according to the sign of $x$, so each element contributes the single condition $x > 0$ or $x < 0$.

\paragraph{Maximum and $\argmax$}
Selecting element $k$ from $\{x_i\}_{i=1}^{m}$ imposes the $m - 1$ conditions $x_k - x_i > 0$ for $i \neq k$; minima reverse the sign.

\paragraph{Sorting and $\argsort$}
A sort that returns the order $x_{\pi(1)} \geq \dots \geq x_{\pi(m)}$ imposes the $m - 1$ consecutive-rank conditions $x_{\pi(j)} - x_{\pi(j+1)} \geq 0$; ties again occur only on a null set.

\paragraph{Equality}
\texttt{==} and \texttt{!=} are left untracked: two rational functions that are not identical agree at finitely many $z$ only, so omitting the condition alters $\gZ$ at most on a null set (which, as in the proof above, leaves $\pselective$ unchanged).
Identical functions, in turn, agree everywhere and impose no constraint.

\paragraph{Degree growth}
Multiplying two ratios adds their degrees, and so does adding ratios whose denominators differ, since $\frac{p}{q} + \frac{u}{v} = \frac{pv + uq}{qv}$; division raises no degree by itself (it merely swaps numerator and denominator) but creates the unequal denominators that make later additions compound.
A pipeline of linear steps therefore stays at degree one.
For the cross-validated $R^2$ score of Section~\ref{subsec:examples}, each fold contributes the ratio of the residual and total sums of squares, both quadratic in $z$, so a $K$-fold score (a sum of $K$ such ratios) is rational of degree $2K$ over $2K$.
Comparing two candidates' scores clears both denominators, which are positive, and leaves a polynomial inequality of degree $4K$.

\section{Details of the Comparison Methods}
\label{app:comparison}
This appendix defines the four baselines of Section~\ref{subsec:setup}.

\begin{itemize}
  \item \texttt{w/o-pp}: an ablation of \texttt{proposed} that runs the user's algorithm once at the observed point $z_{\mathrm{obs}}$ and replaces the truncation region $\gZ$ in Eq.~\eqref{eq:selective} by the single interval $\big( I^{-}(z_{\mathrm{obs}}),\, I^{+}(z_{\mathrm{obs}}) \big)$ of Eq.~\eqref{eq:isosign}, skipping the parametric-programming sweep of Algorithm~\ref{alg:autosi}:
  \begin{equation}
    p_{\text{w/o-pp}}
    = \PP_{\mathrm{H}_0}\!\left( |z| \geq |z_{\mathrm{obs}}| \,\middle|\, z \in \big( I^{-}(z_{\mathrm{obs}}), I^{+}(z_{\mathrm{obs}}) \big) \right).
  \end{equation}
  Because the interval conditions on the algorithm's entire computation path rather than on its output alone, the inference remains valid but is over-conditioned, which costs power~\citep{le2021parametric,duy2022more}.

  \item \texttt{naive}: the unadjusted $p$-value $\pnaive$ of Eq.~\eqref{eq:naive}, which treats the contrast $\veta$ as if it had been fixed before seeing the data and evaluates the test statistic under the unconditional normal law of $\veta^\top \vY$. It ignores the selection and does not control the type I error rate.

  \item \texttt{bonferroni}: a Bonferroni correction of $\pnaive$ over every selection the algorithm could have made; bounding the number of possible selections by the number of feature subsets, $2^d$, the corrected $p$-value is
  \begin{equation}
    p_{\mathrm{bonferroni}} = \min\left(1,\; 2^d \cdot \pnaive\right).
  \end{equation}
  The correction is valid but extremely conservative, because the algorithm could in fact have produced only a small fraction of these selections.

  \item \texttt{ds}: data splitting, which divides the $n$ observations into two interleaved halves by taking every other observation, runs the same selection algorithm on one half, and tests the drawn feature's coefficient in the regression on the selected features, computed on the other half, with a classical two-sided $z$-test~\citep{fithian2014optimal}.
  Under the independence setting the two halves are independent, so the test is valid; in the correlation setting the halves are dependent, and the test does not control the type I error rate.
  Data splitting pays for its simplicity twice: the selection sees only half of the data, so it may differ from the selection on the full data, and the test statistic uses only the remaining half, which reduces power.
\end{itemize}

\section{Details of the Experiments}
\label{app:expdetails}

\paragraph{Hyperparameters}
The lasso minimizes $\frac{1}{2n} \norm{\vy - \mX \vbeta}_2^2 + \lambda \norm{\vbeta}_1$ with $\lambda = 0.1$, and its iterative fit stops when the coefficients change by less than $10^{-4}$ in $\ell_1$ norm or after $1{,}000$ passes over the features.
The cross-validated lasso selects $\lambda$ among three log-spaced candidates on $[10^{-2}, 10^{-1}]$ by the average validation $R^2$ over $K = 3$ folds, drawn as a random partition of the observations that is fixed within each trial.
All hyperparameter values were fixed a priori and not tuned.
In all three methods the tested contrast $\veta$ is the row of $(\mX_{\gM}^\top \mX_{\gM})^{-1} \mX_{\gM}^\top$ corresponding to the drawn feature, so the test statistic $\veta^\top \vy$ is that feature's least-squares coefficient in the regression on the selected features.
The draw is a uniform random choice from the selected set, fixed independently of the response within each trial.
Trials in which the lasso variants select no feature are discarded.
All experiments were executed on an AMD EPYC 9474F CPU ($48$ cores, $3.6$\,GHz) with $768$\,GB of memory, running Ubuntu 24.04 with Python 3.11.
Every trial uses an explicit random seed, so all reported numbers are reproducible.
Each reported rate on synthetic data averages $1{,}000$ binary outcomes and each rate on real data averages $100$, so their standard errors are at most $\sqrt{0.25/1000} \approx 0.016$ and $\sqrt{0.25/100} = 0.05$, respectively.

\paragraph{Real-data protocol}
The seven datasets of Table~\ref{tab:realdata} are all regression tasks from the UCI Machine Learning Repository~\citep{kelly2023uci}, chosen as standard public benchmarks.
For each of the $100$ repetitions per dataset we subsample $n = 100$ observations without replacement and standardize each feature column (dropping columns that are constant within the subsample) and the response.
The noise variance is estimated by $\hat\sigma^2 = \norm{\vy - \mX \hat\vbeta}_2^2 / (n - d - 1)$ from the least-squares fit on the subsample.
The inference is then exact up to this plug-in estimate of the variance.
Because the lasso variants carry no intercept, the algorithm first centers the response; the centering is a linear operation on the data, so it is tracked like any other operation and the inference remains valid.
Table~\ref{tab:realdata-all} reports the rejection rates of Section~\ref{subsec:realdata} for all three feature-selection methods.

\begin{table}[t]
\centering
\small
\setlength{\tabcolsep}{4pt}
\begin{tabular}{@{}lccc@{}}
\toprule
Dataset & \texttt{bonferroni} & \texttt{w/o-pp} & \texttt{proposed} \\
\midrule
\multicolumn{4}{@{}l}{\emph{Marginal screening}} \\
Airfoil          & \textbf{.83} & .54 & .75 \\
Concrete         & .75 & .35 & \textbf{.75} \\
Energy (heating) & .78 & .85 & \textbf{.91} \\
Energy (cooling) & .66 & .78 & \textbf{.80} \\
Real estate      & .49 & .46 & \textbf{.67} \\
Wine (red)       & .27 & .19 & \textbf{.54} \\
Wine (white)     & .07 & .08 & \textbf{.47} \\
\addlinespace
\multicolumn{4}{@{}l}{\emph{Lasso}} \\
Airfoil          & .79 & .38 & \textbf{.82} \\
Concrete         & .64 & .31 & \textbf{.75} \\
Energy (heating) & .94 & .63 & \textbf{.94} \\
Energy (cooling) & .97 & .54 & \textbf{.97} \\
Real estate      & .40 & .22 & \textbf{.63} \\
Wine (red)       & .17 & .14 & \textbf{.38} \\
Wine (white)     & .21 & .07 & \textbf{.37} \\
\addlinespace
\multicolumn{4}{@{}l}{\emph{Cross-validated lasso}} \\
Airfoil          & .65 & .04 & \textbf{.84} \\
Concrete         & .48 & .07 & \textbf{.56} \\
Energy (heating) & .64 & .08 & \textbf{.70} \\
Energy (cooling) & .68 & .05 & \textbf{.71} \\
Real estate      & .30 & .02 & \textbf{.50} \\
Wine (red)       & .13 & .08 & \textbf{.34} \\
Wine (white)     & .16 & .05 & \textbf{.30} \\
\bottomrule
\end{tabular}
\caption{%
Rejection rates at $\alpha = 0.05$ on seven real datasets for all three feature-selection methods, where a higher rate means more detected associations.
For each method, the highest rate is shown in bold (ties broken toward \texttt{proposed}).
}
\label{tab:realdata-all}
\end{table}

\clearpage
\onecolumn
\section{Complete Code of the Lasso and Cross-Validated Lasso}
\label{app:cvlisting}
Listing~\ref{lst:cv} shows the full user-side code of the cross-validated lasso of Section~\ref{subsec:examples}, comprising the iterative lasso fit and the $K$-fold cross-validation over the candidate grid.
Although this algorithm's selection event, a system of polynomial inequalities up to degree twelve, is far beyond what has been derived by hand, the code contains no inference-specific logic beyond the final \texttt{inference} call.
Listing~\ref{lst:lasso} shows the corresponding self-contained code of the lasso at the fixed regularization strength $\lambda$.

\begin{listing}[H]
\begin{lstlisting}
import numpy as np
import autosi as asi

# X: design matrix wrapped by asi.array
# folds: list of K index arrays
# grid: candidate regularization strengths
# n, d: numbers of observations and features
# max_iter, tol: stopping rule of the fit
# idx: tested position in the selected set
#      (random draw in the experiments)
def fit_lasso(y, X, n, lam):
    msq = asi.sum(X ** 2, axis=0) / n
    beta = asi.array(np.zeros(d))
    for _ in range(max_iter):
        old = beta.copy()
        for j in range(d):
            r = y - X @ beta + X[:, j] * beta[j]
            t = X[:, j] @ r / n
            if t > lam:
                beta[j] = (t - lam) / msq[j]
            elif t < -lam:
                beta[j] = (t + lam) / msq[j]
            else:
                beta[j] = asi.array(0.0)
        if asi.sum(asi.abs(beta - old)) < tol:
            break
    return beta

def cv_r2_lasso(y):
    scores = []
    for lam in grid:
        r2 = asi.array(0.0)
        for k in range(K):
            tr = np.concatenate([folds[j] for j in range(K) if j != k])
            va = folds[k]
            beta = fit_lasso(y[tr], X[tr], len(tr), lam)
            res = y[va] - X[va] @ beta
            ss_res = asi.sum(res ** 2)
            mean = asi.sum(y[va]) / len(va)
            ss_tot = asi.sum((y[va] - mean) ** 2)
            r2 = r2 + 1.0 - ss_res / ss_tot
        scores.append(r2 / K)
    best = asi.argmax(asi.stack(scores))
    beta = fit_lasso(y, X, n, grid[best])
    X_M = X.T[beta != 0]
    eta = (X_M @ X_M.T).inv() @ X_M
    return eta[idx]

eta = cv_r2_lasso(y)
result = asi.inference(eta=eta, prob_vec=y, var=1.0, algorithm=cv_r2_lasso)
\end{lstlisting}
\caption{%
The complete cross-validated lasso.
The code reads as an ordinary implementation of the lasso fit and $K$-fold cross-validation; its degree-twelve selection event is derived automatically when \texttt{inference} runs it along the line.
}
\label{lst:cv}
\end{listing}

\begin{listing}[H]
\begin{lstlisting}
import numpy as np
import autosi as asi

# X: design matrix wrapped by asi.array
# n, d: numbers of observations and features
# lam: regularization strength
# max_iter, tol: stopping rule of the fit
# idx: tested position in the selected set
#      (random draw in the experiments)
def lasso(y):
    msq = asi.sum(X ** 2, axis=0) / n
    beta = asi.array(np.zeros(d))
    for _ in range(max_iter):
        old = beta.copy()
        for j in range(d):
            r = y - X @ beta + X[:, j] * beta[j]
            t = X[:, j] @ r / n
            if t > lam:
                beta[j] = (t - lam) / msq[j]
            elif t < -lam:
                beta[j] = (t + lam) / msq[j]
            else:
                beta[j] = asi.array(0.0)
        if asi.sum(asi.abs(beta - old)) < tol:
            break
    X_M = X.T[beta != 0]
    eta = (X_M @ X_M.T).inv() @ X_M
    return eta[idx]

eta = lasso(y)
result = asi.inference(eta=eta, prob_vec=y, var=1.0, algorithm=lasso)
\end{lstlisting}
\caption{%
The complete lasso at a fixed regularization strength $\lambda$.
}
\label{lst:lasso}
\end{listing}


\newpage
\bibliographystyle{plainnat}
\bibliography{ref}

\end{document}